\documentclass[letterpaper]{article} 
\usepackage{aaai24}  
\usepackage{times}  
\usepackage{helvet}  
\usepackage{courier}  
\usepackage[hyphens]{url}  
\usepackage{graphicx} 
\usepackage{natbib}  
\usepackage{caption} 
\usepackage{algorithm}
\usepackage{algorithmic}

\usepackage{algorithm}
\usepackage{algorithmic}
\usepackage{amsmath}
\usepackage{amssymb}
\usepackage{bbding}
\usepackage{multirow}
\usepackage{booktabs}

\usepackage{newfloat}
\usepackage{listings}
\DeclareCaptionStyle{ruled}{labelfont=normalfont,labelsep=colon,strut=off} 
\floatstyle{ruled}
\newfloat{listing}{tb}{lst}{}
\floatname{listing}{Listing}
\newcommand{\gright}{\textcolor[RGB]{8,156,86}{\Checkmark}}
\newcommand{\rnot}{\textcolor{red}{\XSolidBrush}}
\newcommand{\dtn}{\mathcal{D}^\mathcal{T}}
\usepackage{color} 

\usepackage{newfloat}
\usepackage{listings}
\DeclareCaptionStyle{ruled}{labelfont=normalfont,labelsep=colon,strut=off} 
\floatstyle{ruled}
\newfloat{listing}{tb}{lst}{}
\floatname{listing}{Listing}
\usepackage{booktabs}
\usepackage{amsmath}
\usepackage{amssymb}
\usepackage{amsfonts}
\usepackage{stmaryrd}
\usepackage{enumitem}
\usepackage{multirow}
\usepackage{dsfont}

\usepackage{color}
\definecolor{grey}{rgb}{0.5,0.5,0.5}

\title{SCMix: Stochastic Compound Mixing for Open Compound Domain Adaptation in Semantic Segmentation}

\author{
    Kai Yao\equalcontrib\textsuperscript{\rm 1,}\textsuperscript{\rm 2}, Zhaorui Tan\equalcontrib\textsuperscript{\rm 1,}\textsuperscript{\rm 2}, Zixian Su\textsuperscript{\rm 1,}\textsuperscript{\rm 2}, Xi Yang\textsuperscript{\rm 2}, Jie Sun\textsuperscript{\rm 2}, Kaizhu Huang\textsuperscript{\rm 3}\thanks{Corresponding author}}
\affiliations{
    \textsuperscript{\rm 1}Department of Intelligent Science, Xi’an Jiaotong-Liverpool University\\
    \textsuperscript{\rm 2}Department of Computer Science, University of Liverpool \\
    \textsuperscript{\rm 3} Data Science Research Center, Duke Kunshan University
}

\usepackage{bibentry}

\begin{document}

\maketitle

\begin{abstract}
Open compound domain adaptation (OCDA) aims to transfer knowledge from a labeled source domain to a mix of unlabeled homogeneous compound target domains while generalizing to open unseen domains. Existing OCDA methods solve the intra-domain gaps by a divide-and-conquer strategy, which divides the problem into several individual and parallel domain adaptation (DA) tasks. Such approaches often contain multiple sub-networks or stages, which may constrain the model's performance. In this work, starting from the general DA theory, we establish the generalization bound for the setting of OCDA. Built upon this, we argue that conventional OCDA approaches may substantially underestimate the inherent variance inside the compound target domains for model generalization. We subsequently present \textbf{Stochastic Compound Mixing (SCMix)}, an augmentation strategy with the primary objective of mitigating the divergence between source and mixed target distributions. We provide theoretical analysis to substantiate the superiority of SCMix and prove that the previous methods are sub-groups of our methods. Extensive experiments show that our method attains a lower empirical risk on OCDA semantic segmentation tasks, thus supporting our theories. Combining the transformer architecture, SCMix achieves a notable performance boost compared to the SoTA results.
\end{abstract}

\section{Introduction}

Despite the notable success of deep-learning-based methods in semantic segmentation, these methods often demand a considerable amount of pixel-wise annotated data.  To mitigate the cost associated with data collection and annotation, synthetic datasets~\cite{GTA,synthia} have been suggested as an alternative. However, models trained on synthetic data tend to struggle with poor generalization to real images.  
To tackle this challenge, unsupervised domain adaptation (UDA)~\cite{dacs,daformer,proda} has been proposed to transfer 
knowledge from labeled source domains to unlabeled target images. UDA aims to bridge the gap between domains, enabling models to generalize effectively to new and unseen data.
Previous methods~\cite{adaptsegnet,hoffman2018cycada} commonly rely on the assumption that the target data is derived from a single homogeneous domain. However, this may produce suboptimal results when the target data is composed of various subdomains.
Towards a realistic DA setting, Liu et al.~\cite{cdas} proposed the concept of Open Compound Domain Adaptation (OCDA), which incorporates mixed domains in the target without domain labels. This strategy aims to enhance model generalization by adapting to a compound target domain, resulting in better performance when faced with unseen domains.

\begin{figure}[t]
\centering
\includegraphics[width=0.9\linewidth]{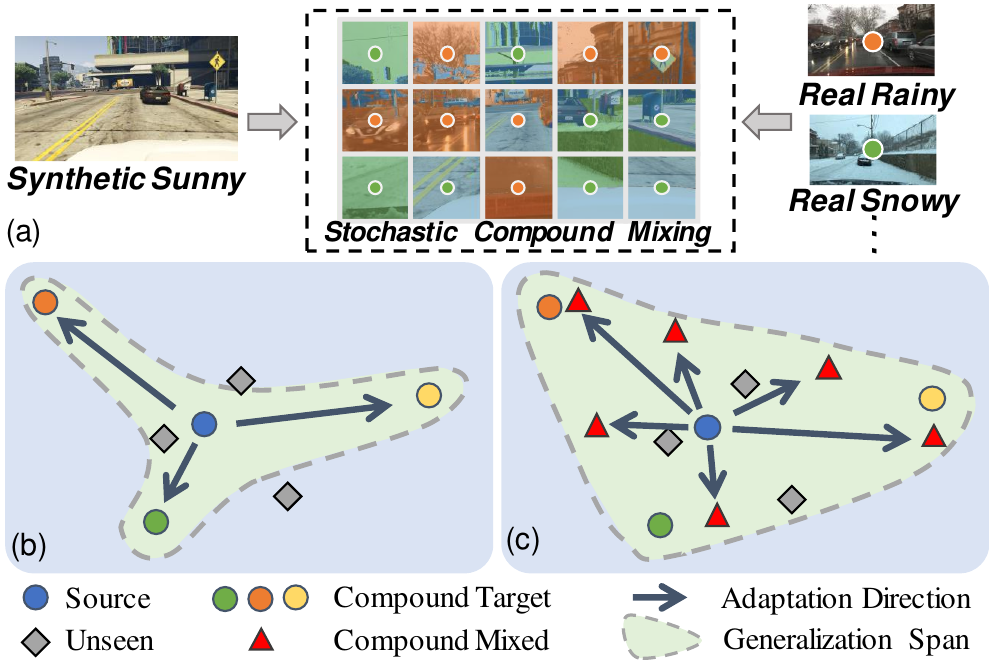}
\caption{
(a) The proposed Stochastic Compound Mixing (SCMix).
(b) Existing works adapt to each target domain iteratively.  
(c)  Our approach focuses on mixing compound domains to enhance the model's adaptation and generalization performance.
}
\label{fig:banner}
\end{figure}

Recent OCDA works mainly take the divide-and-conquer criterion~\cite{csfu,dha,mlbpm}, which first divides the target compound domain into multiple subdomains, and then adapts to each domain respectively. This process equals breaking the complex problem into multiple easier single-target DA problems. To separately align with different target domains, they often contain many sub-networks or stages, and the learning process of the model is equivalent to the ensemble of multiple subdomain models, which may result in trivial solutions in practice and limit their generalization performance (see Fig.~\ref{fig:banner}(b)).

To address the above-mentioned issues, there is a need to investigate how the non-single mixed target domain would affect the overall generalization performance and redefine the adaptation risk. Classical domain adaptation bound analysis~\cite{ben2010theory} considers three factors in upper bounding the target risk: source risk, distribution discrepancy, and the minimal combined risk. Inspired by the analysis, when adapting to the OCDA setting where the target domain is a mixture of multiple subdomains, we derive a new bound showing that the distribution discrepancy should be calculated as \textit{the sum of  discrepancies between the source domain and any possible joint target subdomains}. Through this theory, previous divide-and-conquer methods have simplified this computation by  considering the discrepancy between the source domain and each target subdomain, which ignores the true intra-domain discrepancy and deviates from the definition of OCDA. This simplification may also hinder the model's transfer performance on the target domains.

Based on this finding, we consider that the key challenge lies in leveraging the differences and correlations among compound target domains so as to minimize the actual adaptation gap. To achieve this, we propose \textbf{Stochastic Compound Mixing (SCMix)}, an advanced mixing solution to reduce the divergence between source and mixed target distributions globally (see Fig.~\ref{fig:banner}(c)).
Our strategy stochastically mixes the source images and multiple target domains with a dynamic grid mask, in which we further perform the class-mix process.
We theoretically analyze the advantages of our proposed strategy over the existing methods.
A key contribution of our work is showing that 
SCMix can be considered a generalized extension of the previous single-target mixing, i.e., they are actually a sub-group of our proposed framework according to the group theory framework. This certifies that SCMix yields lower empirical risk.  We conduct extensive experiments to support its effectiveness.
Additionally, considering the robustness of Transformers over CNNs, we first introduce and reveal the significant potential of the transformer backbone in the OCDA setting. 
Built upon this,  our approach outperforms existing state-of-the-art methods by a large margin in both domain adaptation and domain generalization scenarios. 
Our contributions can be summarized as follows:
\begin{itemize}[itemsep=1pt,topsep=1pt,parsep=1pt,leftmargin=12pt]
\item We define the domain adaptation bound for OCDA, which highlights the limitation of the existing methods and illustrates the key to solving this problem. (Sec. 3.1)
\item We introduce a simple yet effective mixing strategy, Stochastic Compound Mixing (SCMix), for OCDA setting, which dynamically mixes the source and multi-target images for global optimal adaptation. (Sec. 3.3)
\item We demonstrate that SCMix generally outperforms existing OCDA methods theoretically. A lower error bound can be achieved with our proposed strategy. (Sec. 3.4) 
\item We conduct large experiments to validate the effectiveness of SCMix, which gives a significant empirical edge on both compound and open domains. (Sec. 4.2 \& 4.3) 

\end{itemize}

\section{Related Work}
\noindent\textbf{Unsupervised Domain Adaptation (UDA)}  techniques aim to overcome the challenges associated with labeling large-scale datasets, which can be cumbersome and expensive. 
Adversarial learning~\cite{adaptsegnet,hoffman2018cycada,fda} is 
a popular strategy in UDA, where the model 
aligns the distribution between the source and target domains at either the image-level or feature-level, thereby implicitly measuring distribution shift and learning domain-invariant representations.
In addition, self-training approaches~\cite{dacs,proda,corda,metacorrection} showed
promising performance 
by utilizing the model trained on labeled source domain data to generate pseudo labels for the target domain data, and iteratively training the model with these labels.
Existing approaches are limited in practical usage because they assume the target data comes from a single distribution. More generalized techniques are needed to handle the challenge of multiple diverse distributions in real-world scenarios.

\noindent\textbf{Open Compound Domain Adaptation (OCDA)} assumes the target is a compound of multiple homogeneous domains without domain labels rather than a predominantly uni-modal distribution. CDAS~\cite{cdas} first proposed the concept of OCDA, introducing a curriculum domain adaptation strategy to learn from source-similar samples. 
AST~\cite{ast} introduced a cross domain feature stylization and a content-preserving normalization to learn from domain-invariant features. Most current OCDA works~\cite{csfu,mlbpm,dha} solved the intra-domain gaps by separating the compound target domain into multiple subdomains. CSFU~\cite{csfu} adopted domain-specific batch normalization for adaptation. DHA~\cite{dha} 
used GAN-based image translation and adversarial training to extract domain-invariant features from subdomains. ML-BPM~\cite{mlbpm} presented a multi-teacher distillation strategy to learn from multi-subdomains efficiently. Different from previous works that take a divide-and-conquer strategy to reduce the domain gap, we argue that taking advantage of the intra-domain variance will minimize the gap between the source and multiple compound target domains, thus benefiting the domain adaptation and generalization performance.


\noindent\textbf{Mixing} pixels from two training images to create highly perturbed samples has been proven successful in UDA for semantic segmentation~\cite{dacs,daformer,camix}. Mixing technology improves the domain adaptation performance by generating pseudo labels with weak augmented images and training the model with strong augmented images through consistency regularization. For instance, CutMix~\cite{cutmix} cut a rectangular region from one image and pasted it on top of another. ClassMix~\cite{classmix} further developed this line by creating the selection dynamically with the assistance of the ground truth mask, which was later introduced to perform cross-domain mixing~\cite{dacs}. Nonetheless, previous works have concentrated on mixing one single target with one source image, which restricts their efficacy in OCDA tasks. Conversely, we contend that utilizing multi-target mixing could be an effective and straightforward method to further enhance performance. 

\newtheorem{assumption}{Assumption}
\newtheorem{proposition}{Proposition}
\newtheorem{theorem}{Theorem}
\newtheorem{proof}{Proof}
\newtheorem{definition}{Definition}

\section{Method}

\subsection{Theoretical Motivation}

This section establishes the generalization bound for OCDA and provides the theoretical motivation for this paper.

Let the data and label spaces be represented by $\mathcal{X}$ and $\mathcal{Y}$ respectively, and $h:\mathcal{X}\rightarrow\mathcal{Y}$ be a mapping such that $h\in \mathcal{H}$ is a set of candidate hypothesis.
Under OCDA settings, the target domain comprises multiple known and unknown homogeneous subdomains. Its learning bound should consider the relationships between target subdomains. Inspired by the theory of multi-source DA risk bound in~\cite{ben2010theory}, we propose to calculate the $\mathcal{H}\Delta\mathcal{H}$-distance between the source domain and the combination of target subdomains and attain  the OCDA Learning Bound:
\begin{theorem}
\textbf{\rm{(OCDA Learning Bound)}} Let $R^\mathcal{S}$, $R^\mathcal{T}$ be the generalization error on the source domain $\mathcal{D}^\mathcal{S}$ and the target domain $\mathcal{D}^\mathcal{T}$, respectively.  $\mathcal{D}^\mathcal{T}$ 
contains $N$ seen subdomains,such that $\{\mathcal{D}^\mathcal{T}\}_1^N=\{\mathcal{D}^\mathcal{T}_1,\dots,\mathcal{D}^\mathcal{T}_N\}$, and $K-N$ unseen subdomains, $N\ll K$. Given the risk of a hypothesis $h \in \mathcal{H}$, the overall target risk is bounded by:
\begin{equation}
\begin{split}
R^\mathcal{T}(h) \leq R^\mathcal{S}(h) +  \sum\nolimits_{i} \sum\nolimits_{j \geq i}d_{\mathcal{H}\Delta\mathcal{H}}(\mathcal{D}^\mathcal{S},\mathcal{J}_{ij}) + \lambda ,
\end{split}
\end{equation}
where $1 \leq i \leq j \leq  N$, and $\mathcal{J}_{i,j}=\dtn_i \otimes   \dots \otimes \dtn_j$ denotes the joint distribution (joint subdomains) by using the jointing operation $\otimes$.
$d_{\mathcal{H}\Delta\mathcal{H}}$ is the $\mathcal{H}\Delta\mathcal{H}$-distance between $\mathcal{D}^\mathcal{S}$ and $\mathcal{J}$,
\begin{equation}
\begin{split}
d_{\mathcal{H}\Delta\mathcal{H}} \triangleq \sup\limits_{h,h'\in\mathcal{H}}  |&\mathbb{E}_{x \in \mathcal{D}^\mathcal{S}}[h(x)\neq h'(x)] \\
 - &\mathbb{E}_{x \in \mathcal{J}}[h(x)\neq h'(x)]| ,
\end{split}
\end{equation}
and joint hypothesis 
 $\lambda:=\min_{h^*\in\mathcal{H}}(R^\mathcal{S}(h^*)+R^\mathcal{T}(h^*)) $ corresponds to the minimal total risk over all domains.
\end{theorem}
Proof of Theorem~1 is available in the supplement.

If we expand the middle term in the risk, we can derive:
\begin{equation*}
\underbrace{ \overbrace{\sum\nolimits_{i}d_{\mathcal{H}\Delta\mathcal{H}}(\mathcal{D}^\mathcal{S},\mathcal{D}^\mathcal{T}_i)}^{\text{Conventional OCDA risk objective}}+ \sum\nolimits_{i}\sum\nolimits_{j>i}d_{\mathcal{H}\Delta\mathcal{H}}(\mathcal{D}^\mathcal{S},\mathcal{J}_{ij})}_{\text{Our risk objective}} ,
\end{equation*}
where the former part is the objective of conventional OCDA methods (divide-and-conquer strategy) that aims to break down the complex OCDA problem into multiple easier single-target DA problems.
However, this failure to consider 
the disparity between the source and joint probability distributions of the targeted compound domain merely addresses a fraction of the total risk involved.
In contrast, inspired by the observation, we aim to reduce the overall risk by minimizing the gap between the source and the joint distribution of compound target domains. In the following sections, we will first describe the self-training framework and propose our novel augmentations to approximate our risk objective. Then, we shall theoretically establish the superiority of our approach by leveraging insights from group theory.

\subsection{Preliminaries}
Given the paired source domain images with one-hot labels $\{(x^\mathcal{S}_{n_s},y^\mathcal{S}_{n_s})\}^{N_s}_{{n_s}=1}$ with $C$ classes and the unlabeled target compound domain images $\{(x^\mathcal{T}_{n_t})\}^{N_t}_{{n_t}=1}$, we aim to train a segmentation network that achieves promising performance on the target domain. Directly training the network $f$ on source data with categorical cross-entropy (CE) loss cannot guarantee good performance on the target domain due to the domain gap: 
\begin{equation}
\mathcal{L}_{CE}^\mathcal{S}=  -\sum^{W \times H}_{i=1} \sum^{\mathcal{C}}_{c=1} y^{\mathcal{S}}_{(i,c)} \log (f(x^{\mathcal{S}})_{(i,c)}).
\end{equation}

To tackle the domain gap, one popular self-training (ST) solution~\cite{meanteacher} is to utilize a source-trained teacher network $f_{te}$ to generate target domain pseudo labels $\hat{y}^\mathcal{T}$ by the maximum probable class:
\begin{equation}
\begin{split}
\hat{y}^\mathcal{T}_{(i,c)}=\left\{\begin{array}{ll} 1, & {\rm if~~}c = \arg\max_{c'} f_{te} (x^\mathcal{T})_{(i,c')}\\0, & {\rm otherwise}\end{array}\right. .
\end{split}
\end{equation}
The teacher network $f_{te}$ is not updated by gradient backpropagation but the Exponentially Moving Average (EMA) of the student network weights $\theta_f$ after each training step $t$:
\begin{equation}
\theta_{f_{te}}^{t+1} \leftarrow m \theta_{f_{te}}^{t} + (1-m) \theta_f ,
\label{eq:ema}
\end{equation}
where $m$ is the momentum to temporally ensemble the student network. 
The student model is then trained with the strong augmented image and its label with weighted cross-entropy (WCE) loss:
\begin{equation}
\mathcal{L}_{WCE}^\mathcal{T}= -\sum^{W \times H}_{i=1} \sum^{\mathcal{C}}_{c=1} w_{(i)} y_{(i,c)} \log (f(Aug(x)_{(i,c)})),
\end{equation}
where the weight $w$ is the confidence of the pseudo label.

\subsection{Stochastic Compound Mixing}
To further stabilize the training process and minimize the domain gaps, we follow previous UDA works~\cite{dacs,daformer,corda} to generate pseudo labels on non-augmented images, and train the student network with domain-mixed images. Commonly, the process of mixing domains is carried out in a single-target manner. This involves selecting sets of pixels from one source image and pasting them onto a target image. Under OCDA, we argue that mixing one source image with multiple compound target images enhances the model's generalization, as per our theory, which yields advantages for  OCDA. In order to attain our objective, the domain-mixing strategy should meet the following criteria:
\begin{itemize}[itemsep=1pt,topsep=1pt,parsep=1pt,leftmargin=12pt]
\item  It should involve multiple compound target domain images to form multi-target mixing.
\item  It should preserve the local semantic consistency between the source and mixed target images.
\item  It should provide sufficient perturbation to improve the model's robustness against unseen factors.
\end{itemize}

To this end, we propose an augmentation strategy that mixes one source image with multiple target domain images. We stochastically sample multiple target images to conduct compound mixing to cover the possible permutation and combination of the mixture. We then perform class mixing between the source image and the compound-mixed target image in each grid to retain semantic consistency.


\begin{figure}[t]
\centering
\includegraphics[width=0.9\linewidth]{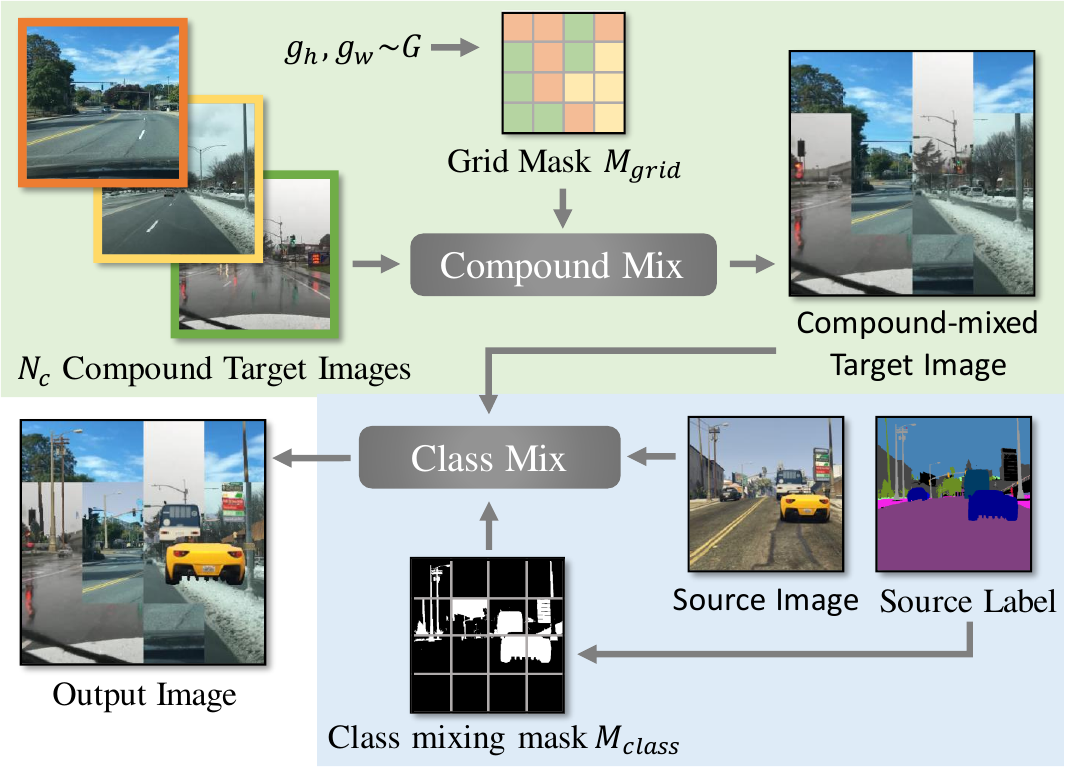}
\caption{Examples augmented using SCMix: an image from the source domain is
mixed with multiple images from the compound target domain. }%
\label{fig:scmix}
\end{figure}

Formally, in each iteration, one source image and $N_c$ compound target images with the corresponding ground truth and pseudo labels are sampled, denoted as $(x^\mathcal{S},y^\mathcal{S}, x^\mathcal{T}_1,\hat{y}^\mathcal{T}_1,\dots, x^\mathcal{T}_{N_c},\hat{y}^\mathcal{T}_{N_c})$. In addition, a confidence estimate $w$ is produced for the pseudo labels, where we use the ratio of pixels exceeding a threshold $\tau$ of the maximum softmax probability~\cite{dacs}:
\begin{equation}
w^\mathcal{T}_{k} = {\textstyle\sum_{n=1}^{W\times H} \llbracket\max_{c'} f_{te}(x^\mathcal{T}_k)_{,(n,c')} > \tau\rrbracket}/{WH} ,
\label{eq:weight}
\end{equation}
where $k\in\{1,...,N_c\}$ is the index of target images, $\llbracket\cdot\rrbracket$ denotes Iverson bracket.

We start by generating a grid mask $M_{grid} \in \mathbb{R}^{W \times H}$ to fuse target domain images. The mask is divided into grids with size $g_h \times g_v$, where $g_h$ and $g_v$ are the horizontal and vertical grid numbers randomly sampled from the candidate set $G$ respectively. For each grid cell $M_{grid (i)} \in \mathbb{R}^{\frac{W}{g_h} \times \frac{H}{g_v}}$, we randomly select a value from an integer set $[1,2,\dots,N_c]$ to represent the index of $N_c$ target samples. 
Then, the compound-mixed target image $x'$ and corresponding pseudo label $\hat{y}'$ as well as weight $w'$ can be fused by grid masking:
\begin{equation}
\begin{split}
\left\{\begin{array}{ll} 
x' = \sum_{k=1}^{N_c} x^\mathcal{T}_k  \odot \llbracket M_{grid}=k\rrbracket \\
\hat{y}' = \sum_{k=1}^{N_c} \hat{y}^\mathcal{T}_k  \odot \llbracket M_{grid}=k\rrbracket \\
w' = \sum_{k=1}^{N_c} w^\mathcal{T}_k  \odot \llbracket M_{grid}=k\rrbracket
\end{array}\right. ,
\end{split}
\label{eq:1}
\end{equation}
where $\odot$ denotes dot product.


Next, we fuse the compound-mixed target image and the source image. For each grid, a subset of classes~\cite{dacs} is randomly selected from $y^\mathcal{S}$ to form the binary class mixing mask $M_{class} \in \{0,1\}^{W\times H}$, where the pixels is 1 if belonging to the subset otherwise 0. Specifically, to ensure that the areas of the source and $N_c$ target images are balanced in the mixed images, we randomly select $\lceil \frac{c_s}{N_c} \rceil$ classes from the ground truth label of the source images, where $c_s$ is the number of classes in the label.
The final mixed image with its label and weight is defined as:
\begin{equation}
\begin{split}
\left\{\begin{array}{ll} 
x^{mix}= x^\mathcal{S} \odot M_{class} + x'  \odot (1-M_{class}) \\
y^{mix}= y^\mathcal{S} \odot M_{class} + \hat{y}'  \odot (1-M_{class})\\
w^{mix}= \mathds{1}\odot M_{class} + w'  \odot (1-M_{class})
\end{array}\right. ,
\end{split}
\label{eq:2}
\end{equation}
where $\mathds{1} \in \mathbb{R}^{W \times H}$ is an all-one weight map for source domain. Finally, the weighted cross-entropy (WCE) loss to train the student model can be rewritten as:
\begin{equation}
\mathcal{L}_{WCE}^\mathcal{T}= -\sum^{W \times H}_{i=1} \sum^{\mathcal{C}}_{c=1} w^{mix}_{(i)} y^{mix}_{(i,c)} \log (f(x^{mix}_{(i,c)})).
\end{equation}

\begin{algorithm}[tb]
\caption{SCMix Algorithm}
\label{alg:algorithm} 
\begin{algorithmic}[1] 
\REQUIRE Source domain and compound target domain datasets $\mathcal{D}^\mathcal{S}$ and $\mathcal{D}^\mathcal{T}$, student network $f$, teacher network $f_{te}$, the number of mixed target images $N_c$, and candidate set for grid size $G$.
\STATE Initialize network parameters $\theta_{f}$ and $\theta_{f_{te}}$.
\FORALL{$t = 1$ to $T$}
\STATE $x^\mathcal{S},y^\mathcal{S} \sim \mathcal{D}^\mathcal{S}$ 
\STATE $x^\mathcal{T}_1,\dots, x^\mathcal{T}_{N_c} \sim \mathcal{D}^\mathcal{T}$ 
\FORALL{$k = 1$ to $N_c$}
\STATE $\hat{y}^\mathcal{T}_k \leftarrow \arg\max( f_{te}(x^\mathcal{T}_{k})$ )
\STATE $w^\mathcal{T}_k \leftarrow Threshold(f_{te}(x^\mathcal{T}_{k}), \tau)$ \hfill $\triangleright$ See Eq.~\ref{eq:weight}
\ENDFOR
\STATE Generate $M_{grid}$ with $N_c$ and $g_h,g_v\sim G$  
\STATE $x^{mix},y^{mix},w^{mix} \leftarrow$ Augmentation by mixing one source image with $N_c$ target images, along with their labels and weights. \hfill $\triangleright$ See Eq.~\ref{eq:1}, Eq.~\ref{eq:2}
\STATE Compute the overall losses.  \hfill $\triangleright$ See Eq.\ref{eq:all}
\STATE Update student network $\theta_{f}$ with the gradient.
\STATE Update teacher network $\theta_{f_{te}}$ with the EMA in Eq.~\ref{eq:ema}.
\ENDFOR
\STATE \textbf{return} $f$
\end{algorithmic}
\end{algorithm}

In summary, both source images and mixed images are used to train the network with the overall objective:
\begin{equation}
\mathcal{L}=\mathcal{L}_{CE}^\mathcal{S}+\mathcal{L}_{WCE}^\mathcal{T}    .
\label{eq:all}
\end{equation}

The SCMix augmented images is illustrated in Fig.~\ref{fig:scmix}, and the complete training process of SCMix is depicted using pseudo code in Algorithm \ref{alg:algorithm} for better understanding.

\subsection{Theoretical Support for SCMix}
We explicate the mechanism of the Stochastic Compound Mixing strategy grounded in group theory. 
we introduce a grand sample space $\chi $, comprised of a source space $\chi^\mathcal{S}$  and $N$ target spaces  $\chi^\mathcal{T}_i$, where $i \in\{1,...,N\}$. The observations $X^\mathcal{S} \in \chi ^\mathcal{S}, X^\mathcal{T}_i \in \chi^\mathcal{T}_i $ are  i.i.d. sampled from a probability distribution $\mathbb{P}^\mathcal{S}, \mathbb{P}^\mathcal{T}_i$, respectively. $\{X^\mathcal{T}\}_1^N = \{X^\mathcal{T}_1, ..., X^\mathcal{T}_N\}$ is explicitly specified to avoid any ambiguity.  
As stated in~\cite{chen2020group},  we adhere to the following assumption:
\begin{assumption}
\label{ass:invariance}
    Source and target data exhibit \textit{exact invariant} to a certain group of transforms $G$ that acts on the sample space, \textit{i.e.}, function $\phi: G \times \chi \to  \chi$, $(g,X) \mapsto \phi (g,X)$, such that $g (X):= \phi (e, X) = X$ for identity element $e \in G$.
    Thus for $g \in G$, $X$ transformed by $g$ has an equality in distribution of itself:
\begin{equation}
\label{eq:eq}
    X =_d g(X).
\end{equation}
\end{assumption}
Specifically, a sub-group of transform $g\in G$ refers to a specific mixing method among all mixing methods.
\begin{proposition}
\label{prop:g_useful}
   For any feasible $g\in G$, there exists a tighter upper bound for performance that exceeds that of bare OCDA models.
\end{proposition}
\begin{proof} 
    Assuming exact invariance holds, let us consider an estimator $\hat{\theta}(X), X \in \chi $ of $\theta_0$ and its $g$ augmented versions $ \hat{\theta}_{g}(X) = \mathbb{E}_{g}\hat{\theta}(gX)$.
    Based on the findings of~\cite{chen2020group}, for any convex loss function $L(\theta_0, \cdot)$, we have:
    \begin{equation}
    \label{eq:eg}
        \mathbb{E} L(\theta_0, \hat{\theta}(X)) \ge \mathbb{E} L(\theta_0, \hat{\theta}_{g}(X)),
    \end{equation}
   which leads to a tight bound for models with $g$.
    \hfill $\Box$
\end{proof}
Proposition~\ref{prop:g_useful} shows that any proper augmentation can potentially enhance the performance of OCDA tasks. 

\begin{table*}[t]
\centering
\footnotesize
\setlength\tabcolsep{1.4 pt}
\begin{tabular}{l|ccccccccccccccccccc|c}
\toprule
\multicolumn{21}{c}{\textbf{GTA5 $\rightarrow$ C-Driving}}               \\\midrule
Method  & Road & SW. & Build & Wall & Fence & Pole & Light & Sign & Veg. & Terrain & Sky & Person & Rider & Car & Truck & Bus & Train & Motor. & Bike & \textbf{mIoU}\\\hline 
No Adaptation &73.4 & 12.5 & 62.8 & 6.0 & 15.8 & 19.4 & 10.9 & 21.1 & 54.6 & 13.9 & 76.7 & 34.5 & 12.4 & 68.1 & 31.0 & 12.8 & 0.0 & 10.1 & 1.9 & 28.3 \\
CDAS &79.1 & 9.4 & 67.2 & 12.3 & 15.0 & 20.1 & 14.8 & 23.8 & 65.0 & 22.9 & 82.6 & 40.4 & 7.2 & 73.0 & 27.1 & 18.3 & 0.0 & 16.1 & 1.5 & 31.4 \\
CSFU &80.1 & 12.2 & 70.8 & 9.4 & 24.5 & 22.8 & 19.1 & 30.3 & 68.5 & 28.9 & 82.7 & 47.0 & 16.4 & 79.9 & 36.6 & 18.8 & 0.0 & 13.5 & 1.4 & 34.9 \\
DHA &79.9 & 14.5 & 71.4 & 13.1 & 32.0 & 27.1 & 20.7 & 35.3 & 70.5 & 27.5 & 86.4 & 47.3 & 23.3 & 77.6 & 44.0 & 18.0 & 0.1 & 13.7 & 2.5 & 37.1 \\
ML-BPM &85.3 & 26.2 & 72.8 & 10.6 & 33.1 & 26.9 & 24.6 & 39.4 & 70.8 & 32.5 & 87.9 & 47.6 & 29.2 & \textbf{84.8} & 46.0 & 22.8 & \textbf{0.2} & 16.7 & \textbf{5.8 }& 40.2 \\
SCMix (ours)& \textbf{86.4} & \textbf{35.7} & \textbf{76.7} & \textbf{30.6} &\textbf{ 36.4} &\textbf{ 31.7 }& \textbf{27.2} & \textbf{40.8 }& \textbf{75.5} & \textbf{32.8} & \textbf{90.4} & \textbf{49.9} & \textbf{40.8} & 81.5 & \textbf{57.8} & \textbf{42.4} & 0.0 & \textbf{36.9} & 2.0 & \textbf{46.1}\\
\bottomrule
\end{tabular}
\caption{Comparison of GTA5 $\rightarrow$ C-Driving adaptation in terms of mIoU(\%). The best result is highlighted in \textbf{bold}.}
\label{tab:gta}
\end{table*}
\begin{table*}[t]
\centering
\footnotesize
\setlength\tabcolsep{2 pt}
\begin{tabular}{l|cccccccccccccccc|c|c}
\toprule
\multicolumn{19}{c}{\textbf{SYNTHIA $\rightarrow$ C-Driving}}               \\\midrule
Method  & Road & SW. & Build & Wall & Fence & Pole & Light & Sign$^*$ & Veg. & Sky & Person & Rider$^*$ & Car & Bus$^*$ & Motor.$^*$ & Bike$^*$ & \textbf{mIoU}& \textbf{mIoU$^*$} \\\hline 
No Aadaptation  &  33.9 & 11.9 & 42.5 & 1.5 & 0.0 & 14.7 & 0.0 & 1.3 & 56.8 & 76.5 & 13.3 & 7.4 & 57.8 & 12.5 & 2.1 & 1.6 & 20.9 & 28.1 \\
CDAS &  54.5 & 13.0 & 53.9 & 0.8 & 0.0 & 18.2 & 13.0 & 13.2 & 60.0 & 78.9 & 17.6 & 3.1 & 64.2 & 12.2 & 2.1 & 1.5 & 25.4 & 34.0 \\
CSFU &  69.6 & 12.2 & 50.9 & 1.3 & 0.0 & 16.7 & 12.1 & 13.6 & 56.2 & 75.8 & 20.0 & 4.8 & 68.2 & 14.1 & 0.9 & 1.2 & 26.1 & 34.8 \\
DHA   &  67.5 & 2.5 & 54.6 & 0.2 & 0.0 & 25.8 & 13.4 & 27.1 & 58.0 & 83.9 & 36.0 & 6.1 & 71.6 & 28.9 & 2.2 & 1.8 & 30.0 & 37.6 \\
ML-BPM&  73.4 & 15.2 & 57.1 & 1.8 & 0.0 & 23.2 & 13.5 & 23.9 & 59.9 & 83.3 & \textbf{40.3} & \textbf{22.3} & 72.2 & 23.3 & 2.3 & 2.2 & 32.1 & 40.0 \\
SCMix (ours) &   \textbf{85.1} & \textbf{34.5} &\textbf{ 74.6} & \textbf{8.4} & \textbf{0.2} &\textbf{ 28.7 }& \textbf{25.1} & \textbf{38.2} & \textbf{72.7} &\textbf{ 90.3} & 18.7 & 5.1 & \textbf{75.0} & \textbf{47.3} & \textbf{10.0} & \textbf{3.9 }& \textbf{38.6 }& \textbf{46.7}\\
\bottomrule
\end{tabular}
\caption{Comparison of SYNTHIA $\rightarrow$ C-Driving adaptation in terms of mIOU(\%). The best result is highlighted in \textbf{bold}. mIOU$^*$ denotes the average of 11 classes (computed without the classes marked with $^*$).}
\label{tab:synthia}
\end{table*}
\begin{table*}[t]
\centering
\footnotesize
\setlength\tabcolsep{2.5 pt}
\begin{tabular}{l|c|ccccc|ccccc}
\toprule
\multirow{2}{*}{Method} & \multirow{2}{*}{Type} & \multicolumn{5}{c|}{GTA5 $\rightarrow$ OpenSet} & \multicolumn{5}{c}{SYNTHIA $\rightarrow$ OpenSet} \\ \cline{3-12} 
 &  & Open & Cityscape & KITTI & WildDash & Avg & Open & Cityscape & KITTI & WildDash & Avg \\ \hline
CDAS & OCDA & 38.9 & 38.6 & 37.9 & 29.1 & 36.1 & 36.2 & 34.9 & 32.4 & 27.6 & 32.8 \\
RobustNet& DG & 38.1 & 38.3 & 40.5 & 30.8 & 36.9 & 37.1 & 38.3 & 40.1 & 29.6 & 36.3 \\
DHA & OCDA & 39.4 & 38.8 & 40.1 & 30.9 & 37.3 & 37.1 & 38.3 & 40.1 & 29.6 & 36.3 \\
SHADE & DG & 38.7 &  46.7 & \textbf{48.3} & 24.7 & 39.6 & - &  -& - &- & - \\
ML-BPM& OCDA & 42.5 & 41.7 & 44.3 & 34.6 & 40.8 & 38.9 & 38.0 & 40.6 & 30.0 & 36.9 \\
SCMix (ours) & OCDA & \textbf{50.7} & \textbf{52.2} & 48.0 & \textbf{49.3 }& \textbf{50.1} & \textbf{50.1} & \textbf{59.4} &\textbf{50.8} &\textbf{ 35.5} & \textbf{49.0} \\ \bottomrule
\end{tabular}
\caption{Comparison of generalization performance on the open set in terms of mIOU(\%). The best is highlighted in \textbf{bold}.}
\label{tab:openset}
\end{table*}

We conducted a theoretical study to assess the superiority of SCMix over single-target mixing strategies.
Consider a group of mixing methods $g_{o}$ which mix source domains with a single target domain and another group $g_{x}$ which uses the SCMix approach, Eq.~\ref{eq:eq} also holds on $X \in \chi^\mathcal{S}, X \in \chi^\mathcal{T}_i$.
Based on Assumption~\ref{ass:invariance}, we have Proposition~\ref{prop:h_ge_g}:
\begin{proposition}
\label{prop:h_ge_g}
   Single-target mixing  ($g_{o} \in G$) is a sub-group of the SCMix ($g_{x} \in G$), i.e., $g_{o} \in g_{x}$. Additionally, the sample space of $g_o$ is a subspace of $g_x$: $\chi_{g_o}$ $\in$ $\chi_{g_x}$.
\end{proposition}
\begin{proof} 
Assuming exact invariance holds, we can consider the operation of $g_o$ as:
\begin{align}
    g_o ( {X}^\mathcal{S} ,  {X}^\mathcal{T}_i) = ({X}^\mathcal{S},{X}^\mathcal{T}_i, {X}^\mathcal{S} \oplus {X}^\mathcal{T}_i), \notag
\end{align}
where $\oplus$ denotes the augmented filled sample spaces in between. 
Thus, for all target domains, we can write: 
\begin{equation} 
\small
\label{eq:x_n}
    g_o ({X}^\mathcal{S} , \{X^\mathcal{T}\}_1^N) = ({X}^\mathcal{S} ,\{X^\mathcal{T}\}_1^N, {X}^\mathcal{S} \oplus \{X^\mathcal{T}\}_1^N ). 
\end{equation}

For $g_x$ operation, we have:
\begin{equation}
\small
\label{eq:Cx_n}
g_x ({X}^\mathcal{S} , \{X^\mathcal{T}\}_1^N )  = ({X}^\mathcal{T} ,\{X^\mathcal{T}\}_1^N,{X}^\mathcal{S} \oplus {C}(\{X^\mathcal{T}\}_1^N)), 
\end{equation}
where ${C}$ donates compounding target sample spaces and ${C}(X^\mathcal{T}_i) =\{ X^\mathcal{T}_i,  \{{X}^\mathcal{T}_i |_\oplus \}_i^N \}$. Thus for $\{X^\mathcal{T}\}_1^N$:
\begin{equation}
\small
\begin{split}
    {C}(\{X^\mathcal{T}\}_1^N)) \! =\!  \sum_{i =1}^N \! {C}(\! X^\mathcal{T}_i\! ) \!  = \! \{ \! \{ \! X^\mathcal{T}\! \}_1^N, \! \underbrace{...}_{\frac{n^2 -N}{2} \text{ terms omitted.}}\!   \}   \! \ni\!  \{\! X^\mathcal{T\! }\}_1^N\!  . \notag 
\end{split}
\end{equation}
See omitted terms in Supplementary Materials.
Since $ g_o ({X}^\mathcal{S} , \{X^\mathcal{T}\}_1^N ) \in g_x ({X}^\mathcal{S} , \{X^\mathcal{T}\}_1^N )$, we can infer that $g_o$ is a sub-group of $g_x$ and $\chi _{g_o} \in  \chi _{g_x}$. 
\hfill $\Box$
\end{proof}


Proposition~\ref{prop:h_ge_g} suggests that the proposed SCMix can be seen as an extension of the single-target mixing strategies.  As $N$ increases, the proportion of $g_o$ w.r.t. 
$g_x$ decreases. 
This signifies that $g_x$ will have a more significant impact when $N$ increases. Lastly, we discuss the impact of $N$.
\begin{proposition} 
\label{prop:h_e_g}
Given $N=1$, $g_x$ reduces to $g_o$. For $N\geq2$,  a tight upper bound exists for performance when augmented with $g_x$ rather than $g_o$.
\end{proposition}
\begin{proof}
    When $n=1$, we can easily get that Eq.~\ref{eq:x_n} is equivalent to Eq.~\ref{eq:Cx_n}: $g_x = g_o$.

    When $N\geq2$, assuming exact invariance holds, we consider an estimator $\hat{\theta}(X), X \in \chi $ of $\theta_0$ and its $g_o, g_x$ augmented versions $ \hat{\theta}_{g_o}(X) = \mathbb{E}_{g_o}\hat{\theta}(g_oX), \hat{\theta}_{g_x}(X) = \mathbb{E}_{g_x}\hat{\theta}(g_xX)$.
    Based on \cite{chen2020group}, for any convex loss function $L(\theta_0, \cdot)$, the following holds:
    \begin{align}
    \label{eq:eg}
        \mathbb{E} L(\theta_0, \hat{\theta}(X)) \ge \mathbb{E} L(\theta_0, \hat{\theta}_{g_o}(X)) , X \in \chi_{g_o}, \\
    \label{eq:eh}
        \mathbb{E} L(\theta_0, \hat{\theta}(X)) \ge \mathbb{E} L(\theta_0, \hat{\theta}_{g_x}(X)), X \in \chi_{g_x}.
    \end{align}
     As  $g_o$ is a proper sub-group of $g_x$ when $N \ge 2$, the feature space of $g_o$ is a proper subspace of $g_x$. 
     Based on this, we can conclude the following:
    \begin{align}
        \mathbb{E} L(\theta_0, \hat{\theta}(X)) \ge \mathbb{E} L(\theta_0, \hat{\theta}_{g_o}(X)) \ge \mathbb{E} L(\theta_0, \hat{\theta}_{g_x}(X)) . \notag
    \end{align}
    Thus, using $g_x$ leads to a lower error than $g_o$.
    \hfill $\Box$
\end{proof}
Proposition~\ref{prop:h_e_g} shows how $g_o$ is a special case of $g_x$ when $N=1$. It also indicates that $g_x$ yields better results than $g_o$ when $N\geq 2$. Derived from the above propositions, our proposed SCMix is an empirical $g_x$ to obtain $ {C}(\{X_t\}_1^N)$.

\section{Experiments}

\subsection{Datasets and Implementation Details}
\noindent \textbf{Datasets.} We evaluate our method on a popular scenario, which transfers the information from a synthesis domain to a real one. For the synthesis domain, we use either GTA5 dataset~\cite{GTA} containing 24,966 images with resolution of $1,914 \times 1,052$, or the SYNTHIA dataset~\cite{synthia} consisting of 9,400 images with resolution of $1,280 \times 720$. For the real target domain, we use C-Driving dataset which consists of images with resolution of $1,280 \times 720$ collected from different weather conditions. In particular, it contains 14,697 rainy, snowy, and cloudy images as compound target domain and 627 overcast images as open domain. We further introduce Cityscapes~\cite{cityscapes}, KITTI~\cite{kitti}, and WildDash~\cite{wilddash} along with the open domain as the OpenSet to evaluate the generalization ability on the unseen domains.

\noindent \textbf{Implementation Details.}
Inspired by the recent SoTA UDA setting~\cite{daformer,xie2023sepico,camix,chen2022pipa}, we first introduce transformer-based framework into OCDA task, which consists of a SegFormer MiT-B5~\cite{xie2021segformer} backbone pre-trained on ImageNet-1k~\cite{imagenet} and an ASPP layer~\cite{daformer} with dilation rates of [1, 6, 12,  18]. The output map is up-sampled and operated by a softmax layer to match the input size. AdamW~\cite{adamw} is the optimizer with a learning rate of $6 \times 10^{-5}$ for the backbone and $10 \times$ larger for the rest. Warmup~\cite{daformer} strategy is leveraged in the first 1,500 iterations.
We train the model for 40k iterations on a single NVIDIA RTX 4090 GPU with a batch size of 2. The exponential moving average parameter of the teacher network is 0.999. Following DACS~\cite{dacs}, we utilize the same data augmentation after mixing, including color jitter and Gaussian blur, while $m$ and $\tau$ are set to 0.999 and 0.968, respectively. For the SCMix, $N_c=3$ and $G=[2,4,8]$ are set by default for all the experiments.


 
\subsection{Comparison with Domain Adaptation}
We comprehensively compare the adaptation performance of our approach with existing state-of-the-art OCDA approaches on GTA5 $\rightarrow$ C-Driving. Among them, CDAS~\cite{cdas} is the first work for OCDA. CSFU~\cite{csfu}, DHA~\cite{dha} and ML-BPM~\cite{mlbpm} all adapt subdomain separation. CSFU~\cite{csfu} engages a GAN framework, while DHA~\cite{dha} further introduces a multi-discriminator to minimize the domain gaps. ML-BPM~\cite{mlbpm} employs a self-training framework and multi-teacher distillation. We also provide the non-adapted results, tagged as ``No Adaptation'', which serves as the baseline for this task.

Tab.~\ref{tab:gta} illustrates the adaptation results on task GTA5 $\rightarrow$ C-Driving. By introducing the novel cross-compound mixing strategy to improve the domain generalization performance, the proposed method achieves the state-of-the-art mIoU of 46.1\%. This yields an improvement of 5.9\% compared with the second-best method,  ML-BPM~\cite{mlbpm}. Our method significantly outperforms the previous works by greatly improving some hard classes, e.g., `Rider' and `Bus'. The comparison on task SYNTHIA $\rightarrow$ C-Driving is shown in Tab.~\ref{tab:synthia}. We calculate the mIoU results of 16 categories as well as 11 categories following the previous works. The proposed method achieves the best results, with mIoU of 38.6\% for 16 categories and 46.7\% for 11 categories. By leveraging transformer-based networks and a novel mixing approach supported by our theory, we have significantly improved the results on this task by a large margin, leading to a new benchmark on the OCDA tasks.


\begin{table}[!h]
\centering
\footnotesize
\setlength\tabcolsep{3 pt}
\begin{tabular}{l|l|ccc|ccc}
\hline
\multirow{2}{*}{Architecture} & \multirow{2}{*}{Method} & \multicolumn{3}{c|}{GTA5$\rightarrow$C-Driving} & \multicolumn{3}{c}{SYN$\rightarrow$C-Driving} \\ \cline{3-8} 
 &  & C & O & O+C & C & O & O+C \\ \hline
\multirow{3}{*}{DeepLabv2} & DACS & 36.6 & 39.7 & 38.2 & 36.5 & 36.8 & 36.7 \\
 & ML-BPM & \textbf{40.2} & 40.8 & 40.5 & 40.0 & 36.9 & 38.5 \\
 & SCMix & 39.6 & \textbf{42.8} & \textbf{41.2} & \textbf{40.6} & \textbf{38.2} & \textbf{39.4} \\ \hline
\multirow{2}{*}{Swin+ASPP} & DACS & 36.2 & 40.6 & 38.4 &  40.1 & 43.0 & 41.6 \\
 & SCMix &\textbf{ 43.7 }& \textbf{47.0} & \textbf{45.4} & \textbf{43.5} & \textbf{46.4} & \textbf{45.0} \\ \hline
\multirow{3}{*}{SegFormer} & DACS & 42.1 & 46.3 & 44.2 & 43.6 & 47.3 & 45.5 \\
& DAFormer& 43.3 & 47.0 & 45.2 & 43.1 & 46.2 & 44.7 \\
 & SCMix & \textbf{46.1 }&\textbf{ 50.7} & \textbf{48.4} & \textbf{46.7} & \textbf{50.1} & \textbf{48.4} \\ \hline
\end{tabular}
\caption{Comparison on different architectures. The comparison presents the evaluation performance on the compound domains (C) and open domain (O) of the C-Driving dataset.}
\label{tab:archi}
\end{table}

\subsection{Comparison with Domain Generalization}
We evaluate the domain generalization of the proposed approach against existing OCDA approaches~\cite{cdas,dha,mlbpm}. We also include the latest domain generalization (DG) approaches, e.g., RobustNet~\cite{robustnet} and SHADE~\cite{shade}. We trained all the OCDA approaches with labeled source images and unlabeled compound target images on the C-Driving dataset, while the DG methods were only trained with the source domain. Tab.~\ref{tab:openset} shows the comparative results on the tasks of GTA5 $\rightarrow$ OpenSet and SYNTHIA $\rightarrow$ OpenSet. Though never having seen the target domains, SHADE achieved very promising average results of 39.6 on the GTA5 $\rightarrow$ OpenSet. On the other hand, considering the ease of acquisition of real images and unpredictable weather conditions, OCDA methods can achieve broader applicability than DG methods. For instance, the second-best OCDA method ML-BPM still outperforms the DG methods, indicating the necessity of developing OCDA methods. Our approach outperforms all of the listed OCDA approaches and the DG approach in the table. It has achieved significant improvements compared to the second-best OCDA method, with gains of 9.3\% and 12.1\% achieved on two tasks, respectively. The performance improvement of our method confirms our theoretical claim that mixing the compound targets is beneficial for domain generalization and provides evidence of the effectiveness of SCMix.

\subsection{Comparison of Network Architecture}
We have further evaluate the versatility and robustness of our proposed SCMix method across various architectures. Specifically, we compare our approach against three distinct architectures: DeepLabv2~\cite{deeplab}, Swin+ASPP~\cite{liu2021swin}, and SegFormer. As detailed in Tab.~\ref{tab:archi}, SCMix consistently outperforms other methods in both GTA5$\rightarrow$C-Driving and SYNTHIA$\rightarrow$C-Driving scenarios, regardless of the underlying architecture.
In the context of DeepLabv2, while ML-BPM offers a modest enhancement in results for the open domain, SCMix delivers a considerably more substantial boost. 
With the SegFormer architecture, a notable improvement is observed in the open set, reinforcing our theoretical promise for the OCDA task.

\subsection{Analytical Study}
To evaluate the proposed SCMix and better understand its contribution, we perform the following experiments.

\noindent$\bullet$\quad\textbf{Baseline.} As listed in Tab.~\ref{tab:mix}, even a powerful SegFormer struggle with domain gaps, indicating the necessity of domain adaptation.  The mean teacher strategy~\cite{meanteacher}, effective for self-supervised models, can worsen the performance by 3.5 and 3.0 mIoU in certain domains due to incorrect pseudo labels from the domain gap. Thus, a domain-mixed approach is crucial for stable self-training.


\begin{figure}[t]
\centering
\includegraphics[width=0.92\linewidth]{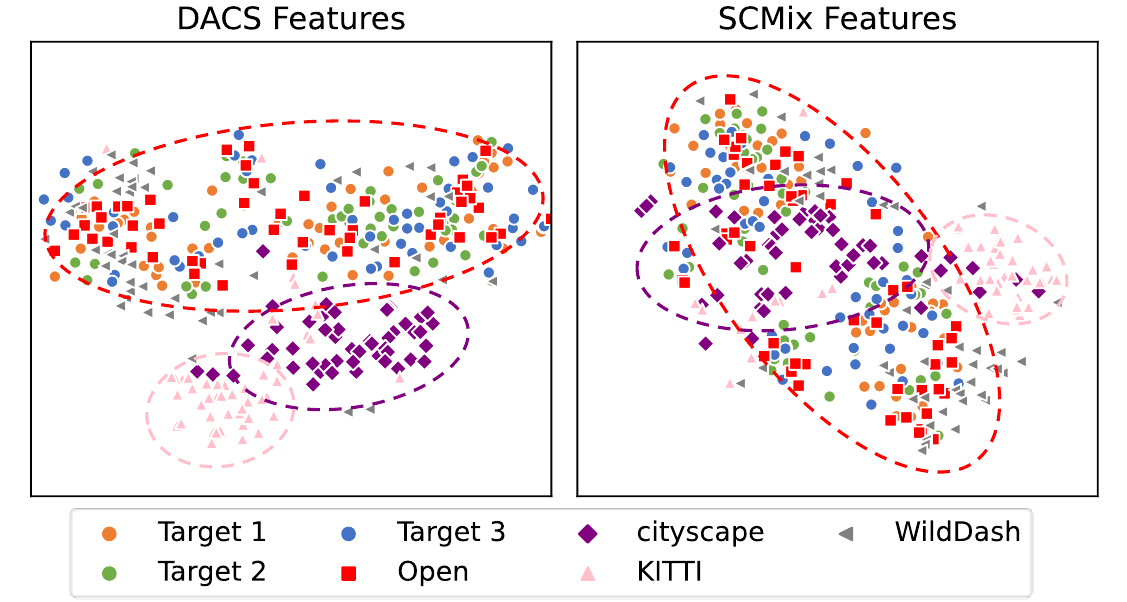}
\caption{T-SNE embedding of backbone features by adapting DACS and SCMix on the target and unseen domains. }
\label{fig:tsne}%
\end{figure}

\begin{table}[t]
\centering
\small
\setlength\tabcolsep{7 pt}
\begin{tabular}{l|cc}
\toprule
Method & GTA5$\rightarrow$C & GTA5$\rightarrow$O \\ \midrule
SegFormer Network& 32.0 (+0.0) & 33.9 (+0.0) \\
+ Mean Teacher (MT)& 28.5 (-3.5) & 30.9 (-3.0) \\\hline
+ MT + CutMix & 40.3 (+8.3) & 44.3 (+10.4) \\
+ MT + CowMix & 40.1 (+8.1) & 44.8  (+10.9) \\
+ MT + fMix & 41.8 (+9.8) & 45.6 (+11.7) \\
+ DACS & 42.1 (+10.1) & 46.3 (+12.4) \\ 
+ SCMix & 46.1 (+\textbf{14.1}) & 50.7 (+\textbf{16.8}) \\ \bottomrule
\end{tabular}
\caption{Comparison with different mixing strategy.}
\label{tab:mix}
\end{table}

\noindent$\bullet$\quad\textbf{Comparison with Different Mixing.} We compare the proposed multi-target mixing algorithm SCMix with four single-target mixing algorithms in Tab.~\ref{tab:mix}. CutMix~\cite{cutmix} cuts out a rectangular region from a source image and pastes it onto a target domain image, achieving a significant performance increase of +8.3 and +10.4 in the target and open domains. CowMix~\cite{cowmix} produces a stronger perturbation, but harms local semantic consistency, limiting its performance improvement. fMix~\cite{fmix} generates masks of arbitrary shape with large connected areas, improving CutMix by +1.7 and +1.3 on the compound target and open domains, respectively. Although DACS maintains semantic consistency with labels, its performance improvement is limited. However, SCMix significantly outperforms CutMix by 6.0 and 6.2 mIoU on compound target and open domains, respectively, nearly twice the improvement achieved by DACS. 

To gain more insights into SCMix on the improved generalization for the unseen domains, Fig.~\ref{fig:tsne} shows the backbone features of target and unseen domains. It is evident that SCMix exhibits a greater overlap between the target and unseen domain distributions, indicating better generalization and robustness of SCMix against unseen factors.



\section{Conclusion}
In this paper, we define a generalization bound for the OCDA task and analyze the limitations of previous divide-and-conquer methods. Upon this, we propose Stochastic Compound Mixing (SCMix) as a novel and efficient method to reduce the divergence between source and mixed target distributions. The superiority of SCMix is supported by theoretical analysis, which proves that SCMix can be considered a generalized extension of single-target mixing and has a lower empirical risk. The effectiveness of the method is validated on two standard benchmarks quantitatively.

\bibliography{aaai24}

\begin{thebibliography}{37}
\providecommand{\natexlab}[1]{#1}

\bibitem[{Abu~Alhaija et~al.(2018)Abu~Alhaija, Mustikovela, Mescheder, Geiger, and Rother}]{kitti}
Abu~Alhaija, H.; Mustikovela, S.~K.; Mescheder, L.; Geiger, A.; and Rother, C. 2018.
\newblock Augmented reality meets computer vision: Efficient data generation for urban driving scenes.
\newblock \emph{Int. J. Comput. Vis.}, 126: 961--972.

\bibitem[{Ben-David et~al.(2010)Ben-David, Blitzer, Crammer, Kulesza, Pereira, and Vaughan}]{ben2010theory}
Ben-David, S.; Blitzer, J.; Crammer, K.; Kulesza, A.; Pereira, F.; and Vaughan, J.~W. 2010.
\newblock A theory of learning from different domains.
\newblock \emph{Machine learning}, 79: 151--175.

\bibitem[{Chen et~al.(2017)Chen, Papandreou, Kokkinos, Murphy, and Yuille}]{deeplab}
Chen, L.-C.; Papandreou, G.; Kokkinos, I.; Murphy, K.; and Yuille, A.~L. 2017.
\newblock Deeplab: Semantic image segmentation with deep convolutional nets, atrous convolution, and fully connected crfs.
\newblock \emph{IEEE Trans. Pattern Anal. Mach. Intell.}, 40(4): 834--848.

\bibitem[{Chen et~al.(2022)Chen, Zheng, Yang, and Chua}]{chen2022pipa}
Chen, M.; Zheng, Z.; Yang, Y.; and Chua, T.-S. 2022.
\newblock PiPa: Pixel-and Patch-wise Self-supervised Learning for Domain Adaptative Semantic Segmentation.
\newblock \emph{arXiv preprint arXiv:2211.07609}.

\bibitem[{Chen, Dobriban, and Lee(2020)}]{chen2020group}
Chen, S.; Dobriban, E.; and Lee, J.~H. 2020.
\newblock A group-theoretic framework for data augmentation.
\newblock \emph{The Journal of Machine Learning Research}, 21(1): 9885--9955.

\bibitem[{Choi et~al.(2021)Choi, Jung, Yun, Kim, Kim, and Choo}]{robustnet}
Choi, S.; Jung, S.; Yun, H.; Kim, J.~T.; Kim, S.; and Choo, J. 2021.
\newblock Robustnet: Improving domain generalization in urban-scene segmentation via instance selective whitening.
\newblock In \emph{IEEE Conf. Comput. Vis. Pattern Recog.}, 11580--11590.

\bibitem[{Contributors(2020)}]{mmseg2020}
Contributors, M. 2020.
\newblock {MMSegmentation}: OpenMMLab Semantic Segmentation Toolbox and Benchmark.
\newblock \url{https://github.com/open-mmlab/mmsegmentation}.

\bibitem[{Cordts et~al.(2016)Cordts, Omran, Ramos, Rehfeld, Enzweiler, Benenson, Franke, Roth, and Schiele}]{cityscapes}
Cordts, M.; Omran, M.; Ramos, S.; Rehfeld, T.; Enzweiler, M.; Benenson, R.; Franke, U.; Roth, S.; and Schiele, B. 2016.
\newblock The cityscapes dataset for semantic urban scene understanding.
\newblock In \emph{IEEE Conf. Comput. Vis. Pattern Recog.}, 3213--3223.

\bibitem[{Deng et~al.(2009)Deng, Dong, Socher, Li, Li, and Fei-Fei}]{imagenet}
Deng, J.; Dong, W.; Socher, R.; Li, L.-J.; Li, K.; and Fei-Fei, L. 2009.
\newblock Imagenet: A large-scale hierarchical image database.
\newblock In \emph{IEEE Conf. Comput. Vis. Pattern Recog.}, 248--255.

\bibitem[{French, Oliver, and Salimans(2020)}]{cowmix}
French, G.; Oliver, A.; and Salimans, T. 2020.
\newblock Milking cowmask for semi-supervised image classification.
\newblock \emph{arXiv preprint arXiv:2003.12022}.

\bibitem[{Gong et~al.(2021)Gong, Chen, Paudel, Li, Chhatkuli, Li, Dai, and Van~Gool}]{csfu}
Gong, R.; Chen, Y.; Paudel, D.~P.; Li, Y.; Chhatkuli, A.; Li, W.; Dai, D.; and Van~Gool, L. 2021.
\newblock Cluster, split, fuse, and update: Meta-learning for open compound domain adaptive semantic segmentation.
\newblock In \emph{IEEE Conf. Comput. Vis. Pattern Recog.}, 8344--8354.

\bibitem[{Guo et~al.(2021)Guo, Yang, Li, and Yuan}]{metacorrection}
Guo, X.; Yang, C.; Li, B.; and Yuan, Y. 2021.
\newblock Metacorrection: Domain-aware meta loss correction for unsupervised domain adaptation in semantic segmentation.
\newblock In \emph{IEEE Conf. Comput. Vis. Pattern Recog.}, 3927--3936.

\bibitem[{Harris et~al.(2020)Harris, Marcu, Painter, Niranjan, Pr{\"u}gel-Bennett, and Hare}]{fmix}
Harris, E.; Marcu, A.; Painter, M.; Niranjan, M.; Pr{\"u}gel-Bennett, A.; and Hare, J. 2020.
\newblock Fmix: Enhancing mixed sample data augmentation.
\newblock \emph{arXiv preprint arXiv:2002.12047}.

\bibitem[{Hoffman et~al.(2018)Hoffman, Tzeng, Park, Zhu, Isola, Saenko, Efros, and Darrell}]{hoffman2018cycada}
Hoffman, J.; Tzeng, E.; Park, T.; Zhu, J.-Y.; Isola, P.; Saenko, K.; Efros, A.; and Darrell, T. 2018.
\newblock Cycada: Cycle-consistent adversarial domain adaptation.
\newblock In \emph{Int. Conf. Mach. Learn.}, 1989--1998.

\bibitem[{Hoyer, Dai, and Van~Gool(2022)}]{daformer}
Hoyer, L.; Dai, D.; and Van~Gool, L. 2022.
\newblock Daformer: Improving network architectures and training strategies for domain-adaptive semantic segmentation.
\newblock In \emph{IEEE Conf. Comput. Vis. Pattern Recog.}, 9924--9935.

\bibitem[{Kundu et~al.(2022)Kundu, Kulkarni, Bhambri, Jampani, and Radhakrishnan}]{ast}
Kundu, J.~N.; Kulkarni, A.~R.; Bhambri, S.; Jampani, V.; and Radhakrishnan, V.~B. 2022.
\newblock Amplitude spectrum transformation for open compound domain adaptive semantic segmentation.
\newblock In \emph{AAAI}, volume~36, 1220--1227.

\bibitem[{Liu et~al.(2021)Liu, Lin, Cao, Hu, Wei, Zhang, Lin, and Guo}]{liu2021swin}
Liu, Z.; Lin, Y.; Cao, Y.; Hu, H.; Wei, Y.; Zhang, Z.; Lin, S.; and Guo, B. 2021.
\newblock Swin transformer: Hierarchical vision transformer using shifted windows.
\newblock In \emph{Proceedings of the IEEE/CVF international conference on computer vision}, 10012--10022.

\bibitem[{Liu et~al.(2020)Liu, Miao, Pan, Zhan, Lin, Yu, and Gong}]{cdas}
Liu, Z.; Miao, Z.; Pan, X.; Zhan, X.; Lin, D.; Yu, S.~X.; and Gong, B. 2020.
\newblock Open compound domain adaptation.
\newblock In \emph{IEEE Conf. Comput. Vis. Pattern Recog.}, 12406--12415.

\bibitem[{Loshchilov and Hutter(2018)}]{adamw}
Loshchilov, I.; and Hutter, F. 2018.
\newblock Decoupled weight decay regularization.
\newblock In \emph{Int. Conf. Learn. Represent.}

\bibitem[{Olsson et~al.(2021)Olsson, Tranheden, Pinto, and Svensson}]{classmix}
Olsson, V.; Tranheden, W.; Pinto, J.; and Svensson, L. 2021.
\newblock Classmix: Segmentation-based data augmentation for semi-supervised learning.
\newblock In \emph{IEEE Winter Conf. on Applications of Comput. Vis.}, 1369--1378.

\bibitem[{Pan et~al.(2022)Pan, Hur, Lee, Kim, and Kweon}]{mlbpm}
Pan, F.; Hur, S.; Lee, S.; Kim, J.; and Kweon, I.~S. 2022.
\newblock ML-BPM: Multi-teacher Learning with Bidirectional Photometric Mixing for Open Compound Domain Adaptation in Semantic Segmentation.
\newblock In \emph{Eur. Conf. Comput. Vis.}, 236--251. Springer.

\bibitem[{Park et~al.(2020)Park, Woo, Shin, and Kweon}]{dha}
Park, K.; Woo, S.; Shin, I.; and Kweon, I.~S. 2020.
\newblock Discover, hallucinate, and adapt: Open compound domain adaptation for semantic segmentation.
\newblock \emph{Adv. Neural Inform. Process. Syst.}, 33: 10869--10880.

\bibitem[{Paszke et~al.(2019)Paszke, Gross, Massa, Lerer, Bradbury, Chanan, Killeen, Lin, Gimelshein, Antiga et~al.}]{paszke2019pytorch}
Paszke, A.; Gross, S.; Massa, F.; Lerer, A.; Bradbury, J.; Chanan, G.; Killeen, T.; Lin, Z.; Gimelshein, N.; Antiga, L.; et~al. 2019.
\newblock Pytorch: An imperative style, high-performance deep learning library.
\newblock \emph{Advances in neural information processing systems}, 32.

\bibitem[{Richter et~al.(2016)Richter, Vineet, Roth, and Koltun}]{GTA}
Richter, S.~R.; Vineet, V.; Roth, S.; and Koltun, V. 2016.
\newblock Playing for data: Ground truth from computer games.
\newblock In \emph{Eur. Conf. Comput. Vis.}, 102--118.

\bibitem[{Ros et~al.(2016)Ros, Sellart, Materzynska, Vazquez, and Lopez}]{synthia}
Ros, G.; Sellart, L.; Materzynska, J.; Vazquez, D.; and Lopez, A.~M. 2016.
\newblock The synthia dataset: A large collection of synthetic images for semantic segmentation of urban scenes.
\newblock In \emph{IEEE Conf. Comput. Vis. Pattern Recog.}, 3234--3243.

\bibitem[{Tarvainen and Valpola(2017)}]{meanteacher}
Tarvainen, A.; and Valpola, H. 2017.
\newblock Mean teachers are better role models: Weight-averaged consistency targets improve semi-supervised deep learning results.
\newblock \emph{Adv. Neural Inform. Process. Syst.}, 30.

\bibitem[{Tranheden et~al.(2021)Tranheden, Olsson, Pinto, and Svensson}]{dacs}
Tranheden, W.; Olsson, V.; Pinto, J.; and Svensson, L. 2021.
\newblock Dacs: Domain adaptation via cross-domain mixed sampling.
\newblock In \emph{IEEE Winter Conf. on Applications of Comput. Vis.}, 1379--1389.

\bibitem[{Tsai et~al.(2018)Tsai, Hung, Schulter, Sohn, Yang, and Chandraker}]{adaptsegnet}
Tsai, Y.-H.; Hung, W.-C.; Schulter, S.; Sohn, K.; Yang, M.-H.; and Chandraker, M. 2018.
\newblock Learning to adapt structured output space for semantic segmentation.
\newblock In \emph{IEEE Conf. Comput. Vis. Pattern Recog.}, 7472--7481.

\bibitem[{Wang et~al.(2021)Wang, Dai, Hoyer, Van~Gool, and Fink}]{corda}
Wang, Q.; Dai, D.; Hoyer, L.; Van~Gool, L.; and Fink, O. 2021.
\newblock Domain adaptive semantic segmentation with self-supervised depth estimation.
\newblock In \emph{Int. Conf. Comput. Vis.}, 8515--8525.

\bibitem[{Xie et~al.(2023)Xie, Li, Li, Liu, Huang, and Wang}]{xie2023sepico}
Xie, B.; Li, S.; Li, M.; Liu, C.~H.; Huang, G.; and Wang, G. 2023.
\newblock Sepico: Semantic-guided pixel contrast for domain adaptive semantic segmentation.
\newblock \emph{IEEE Trans. Pattern Anal. Mach. Intell.}

\bibitem[{Xie et~al.(2021)Xie, Wang, Yu, Anandkumar, Alvarez, and Luo}]{xie2021segformer}
Xie, E.; Wang, W.; Yu, Z.; Anandkumar, A.; Alvarez, J.~M.; and Luo, P. 2021.
\newblock SegFormer: Simple and efficient design for semantic segmentation with transformers.
\newblock \emph{Adv. Neural Inform. Process. Syst.}, 34: 12077--12090.

\bibitem[{Yang and Soatto(2020)}]{fda}
Yang, Y.; and Soatto, S. 2020.
\newblock Fda: Fourier domain adaptation for semantic segmentation.
\newblock In \emph{IEEE Conf. Comput. Vis. Pattern Recog.}, 4085--4095.

\bibitem[{Yun et~al.(2019)Yun, Han, Oh, Chun, Choe, and Yoo}]{cutmix}
Yun, S.; Han, D.; Oh, S.~J.; Chun, S.; Choe, J.; and Yoo, Y. 2019.
\newblock Cutmix: Regularization strategy to train strong classifiers with localizable features.
\newblock In \emph{Int. Conf. Comput. Vis.}, 6023--6032.

\bibitem[{Zendel et~al.(2018)Zendel, Honauer, Murschitz, Steininger, and Dominguez}]{wilddash}
Zendel, O.; Honauer, K.; Murschitz, M.; Steininger, D.; and Dominguez, G.~F. 2018.
\newblock Wilddash-creating hazard-aware benchmarks.
\newblock In \emph{Eur. Conf. Comput. Vis.}, 402--416.

\bibitem[{Zhang et~al.(2021)Zhang, Zhang, Zhang, Chen, Wang, and Wen}]{proda}
Zhang, P.; Zhang, B.; Zhang, T.; Chen, D.; Wang, Y.; and Wen, F. 2021.
\newblock Prototypical pseudo label denoising and target structure learning for domain adaptive semantic segmentation.
\newblock In \emph{IEEE Conf. Comput. Vis. Pattern Recog.}, 12414--12424.

\bibitem[{Zhao et~al.(2022)Zhao, Zhong, Zhao, Sebe, and Lee}]{shade}
Zhao, Y.; Zhong, Z.; Zhao, N.; Sebe, N.; and Lee, G.~H. 2022.
\newblock Style-hallucinated dual consistency learning for domain generalized semantic segmentation.
\newblock In \emph{Eur. Conf. Comput. Vis.}, 535--552. Springer.

\bibitem[{Zhou et~al.(2022)Zhou, Feng, Gu, Pang, Cheng, Lu, Shi, and Ma}]{camix}
Zhou, Q.; Feng, Z.; Gu, Q.; Pang, J.; Cheng, G.; Lu, X.; Shi, J.; and Ma, L. 2022.
\newblock Context-aware mixup for domain adaptive semantic segmentation.
\newblock \emph{IEEE Trans. Circuit Syst. Video Technol.}

\end{thebibliography}

\clearpage
\appendix

\section{Sensitive Study}
\begin{table}[t]
\centering
\small
\setlength\tabcolsep{3pt}
\begin{tabular}{lccclcc}
\cline{1-3} \cline{5-7}
$N_c$ &  GTA5$\rightarrow$C &  GTA5$\rightarrow$O &  & $G$ &  GTA5$\rightarrow$C &  GTA5$\rightarrow$O \\ \cline{1-3} \cline{5-7} 
1 & 42.7 & 47.6 &  & [1, 2] & 40.3 & 45.7 \\
2 & 45.3 & 49.8 &  & [2, 4] & 45.2 & 49.6 \\
3$^\dag$ & \textbf{46.1} & \textbf{50.7} &  & [4, 8] & 45.3 & 49.3 \\
4 & 45.7 & 50.3 &  & [8, 16] & 42.4 & 47.6 \\
5 & 45.6 & 50.6 &  & [2, 4, 8]$^\dag$ & \textbf{46.1} & \textbf{50.7} \\ \cline{1-3} \cline{5-7} 
\end{tabular}
\caption{Sensitive study. Default parameters marked with $^\dag$. }
\label{tab:ablation}
\end{table}

\noindent\quad\textbf{Parameter Analysis on SCMix.} We analyze the impact of two key factors in SCMix: the number of target images ($N_c$) and the grid set $G$. As shown in the left part of Tab.~\ref{tab:ablation}, when $N_c=1$, SCMix is similar to DACS but still improves slightly due to the grid-wise class mix. When $N_c$ is increased to 3, the performance peaks for both adaptation and generalization, validating the theory. However, further increasing $N_c$ may not lead to continued performance improvement. Regarding the choice of $G$, the right part of Tab.~\ref{tab:ablation} shows that there is a reasonable selection range based on network and training image size. $G$ that is too small or too large can limit the effectiveness of SCMix. The optimal choice is $G=[2, 4, 8]$ under the experimental setup.

\section{Proof of OCDA Bound}

\begin{table*}[t]
\centering
\small
\setlength\tabcolsep{4.5 pt}
\begin{tabular}{cccccc}\toprule
 & Source Data & Source Label & Target Data & Multi Target & Unseen Target \\\midrule
Unsupervised Domain Adaptation&  \gright & \gright  & \gright & \rnot & \rnot \\
Multi-target Domian Adaptation&  \gright & \gright  & \gright  & \gright  &  \rnot\\
Open Compound Domain Adaptation&  \gright & \gright  & \gright  & \gright  & \gright \\\bottomrule
\end{tabular}
\caption{Comparisons of different adaptation settings. \textbf{UDA}: unsupervised domain adaptation, \textbf{MTDA}: multi-target domain adaptation, \textbf{OCDA}: open compound domain adaptation. }
\label{tab:diff}
\end{table*}

We first give out the difference of different setting and their corresponding abbreviations in Tab.\ref{tab:diff}. 
Let the data and label spaces be represented by $\mathcal{X}$ and $\mathcal{Y}$ respectively, and a mapping $h:\mathcal{X}\rightarrow\mathcal{Y}$, such that $h\in \mathcal{H}$ is a set of candidate hypothesis. Existing DA bounds idealize the target domain as all the target domain:
\begin{theorem}
\textbf{\rm{(UDA Learning Bound~\cite{ben2010theory})}} Let $R^\mathcal{S}$, $R^\mathcal{T}$ be the generalization error on the source domain $\mathcal{D}^\mathcal{S}$ and the target domain $\mathcal{D}^\mathcal{T}$, respectively. Given the risk of a hypothesis $h \in \mathcal{H}$, the target risk is bounded by: 
\begin{equation}
R^\mathcal{T}(h) \leq R^\mathcal{S}(h) + d_{\mathcal{H}\Delta\mathcal{H}}(\mathcal{D}^\mathcal{S},\mathcal{D}^\mathcal{T}) + \lambda ,
\end{equation}
where $d_{\mathcal{H}\Delta\mathcal{H}}$ is the $\mathcal{H}\Delta\mathcal{H}$-distance between $\mathcal{D}^\mathcal{S}$ and $\mathcal{D}^\mathcal{T}$,
\begin{equation}
\begin{split}
d_{\mathcal{H}\Delta\mathcal{H}} \triangleq \sup\limits_{h,h'\in\mathcal{H}}  |&\mathbb{E}_{x \in \mathcal{D}^\mathcal{S}}[h(x)\neq h'(x)] \\
 - &\mathbb{E}_{x \in \mathcal{D}^\mathcal{T}}[h(x)\neq h'(x)]| ,
\end{split}
\end{equation}
and $\lambda:=\min_{h^*\in\mathcal{H}}(R^\mathcal{S}(h^*)+R^\mathcal{T}(h^*)) $ corresponds to the minimal total risk over both domains.
\end{theorem}

However, MTDA believes that the target domain should consist of many subdomains. Therefore, for MTDA, DA risk bound can be easily generalized to multi-target subdomains by considering the data available from each target subdomain individually: 
\begin{theorem}
\textbf{\rm{(MTDA Learning Bound~\cite{ben2010theory})}} $\mathcal{D}^\mathcal{T}$ contains $K$ subdomains, such that $\{\mathcal{D}^\mathcal{T}\}_1^K=\{\mathcal{D}^\mathcal{T}_1,\dots,\mathcal{D}^\mathcal{T}_K\}$. Given the risk of a hypothesis $h \in \mathcal{H}$, the overall target risk is bounded by: 
\begin{equation}
R^\mathcal{T}(h) \leq R^\mathcal{S}(h) + \sum\nolimits_{i=1}^{K} d_{\mathcal{H}\Delta\mathcal{H}}(\mathcal{D}^\mathcal{S},\mathcal{D}^\mathcal{T}_i) + \lambda^{*} ,
\end{equation}
\end{theorem}

In contrast, similar to DG, OCDA also consider the open domains, which can be treated as unseen subdomains of the overall target domain. Therefore, the bound should not only consider the discrepancy between source and seen target domains but also between the seen and unseen target domains. Inspired by the theory of multi-source DA risk bound in~\cite{ben2010theory}, we propose to calculate the $\mathcal{H}\Delta\mathcal{H}$-distance between the source domain and the combination of target subdomains and attain the proposed OCDA Learning Bound:
\begin{theorem}
\textbf{\rm{(OCDA Learning Bound)}} $\mathcal{D}^\mathcal{T}$ contains $N$ seen subdomains,such that $\{\mathcal{D}^\mathcal{T}\}_1^N=\{\mathcal{D}^\mathcal{T}_1,\dots,\mathcal{D}^\mathcal{T}_N\}$, and $K-N$ unseen subdomains, $N\ll K$. Given the risk of a hypothesis $h \in \mathcal{H}$, the overall target risk is bounded by:
\begin{equation}
\begin{split}
R^\mathcal{T}(h) \leq R^\mathcal{S}(h) +  \sum_{i} \sum_{j \geq i}d_{\mathcal{H}\Delta\mathcal{H}}(\mathcal{D}^\mathcal{S},\mathcal{J}_{ij}) + \lambda^{**} ,
\end{split}
\end{equation}
where $1 \leq i \leq j \leq  N$, and $\mathcal{J}_{i,j}=\dtn_i \otimes   \dots \otimes \dtn_j$ denotes the joint distribution (joint subdomains) by using the jointing operation $\otimes$.
\end{theorem}

\section{More Mathematical Details}
\setcounter{proposition}{1}
\setcounter{proof}{1}

\begin{proposition}
\label{prop:h_ge_g}
   Single-target mixing  ($g_{o} \in G$) is a sub-group of the SCMix ($g_{x} \in G$), i.e., $g_{o} \in g_{x}$. Additionally, the sample space of $g_o$ is a subspace of $g_x$: $\chi_{g_o}$ $\in$ $\chi_{g_x}$.
\end{proposition}
\begin{proof} 
Assuming exact invariance holds, we can consider the operation of $g_o$ as:
\begin{align}
    g_o ( {X}^\mathcal{S} ,  {X}^\mathcal{T}_i) = ({X}^\mathcal{S},{X}^\mathcal{T}_i, {X}^\mathcal{S} \oplus {X}^\mathcal{T}_i), \notag
\end{align}
where $\oplus$ denotes the augmented filled sample spaces in between. 
Thus, for all target domains, we can write: 
\begin{equation} 
\small
\label{eq:x_n}
    g_o ({X}^\mathcal{S} , \{X^\mathcal{T}\}_1^N) = ({X}^\mathcal{S} ,\{X^\mathcal{T}\}_1^N, {X}^\mathcal{S} \oplus \{X^\mathcal{T}\}_1^N ). 
\end{equation}
For $g_x$ operation, we have:
\begin{equation}
\small
\label{eq:Cx_n}
g_x ({X}^\mathcal{S} , \{X^\mathcal{T}\}_1^N )  = ({X}^\mathcal{T} ,\{X^\mathcal{T}\}_1^N,{X}^\mathcal{S} \oplus {C}(\{X^\mathcal{T}\}_1^N)), 
\end{equation}
where ${C}$ donates compounding target sample spaces and ${C}(X^\mathcal{T}_i) =\{ X^\mathcal{T}_i,  \{{X}^\mathcal{T}_i |_\oplus \}_i^N \}$. Thus for $\{X^\mathcal{T}\}_1^N$:
\begin{equation}
\small
\begin{split}
    {C}(\{X^\mathcal{T}\}_1^N)) = \sum_{i =1}^N {C}(X^\mathcal{T}_i) \notag = \{\{X^\mathcal{T}\}_1^N, \underbrace{...}_{\frac{n^2 -N}{2 } \text{terms omitted.}}  \}   \ni \{X^\mathcal{T}\}_1^N .
\end{split}
\end{equation}

The omitted terms are shown as follows: 
\begin{equation}
\small
\begin{split}
    &  {C}(X^\mathcal{T}_1) =\{ X^\mathcal{T}_1,  \{{X}^\mathcal{T}_1 \oplus {X}^\mathcal{T}_2, ...
   , {X}^\mathcal{T}_1 \oplus {X}^\mathcal{T}_2\oplus ... \oplus {X}^\mathcal{T}_{N} 
   \} \}, \notag \\
    &  {C}(X^\mathcal{T}_2) =\{ X^\mathcal{T}_2,  \{{X}^\mathcal{T}_2 \oplus {X}^\mathcal{T}_3, ...
   , {X}^\mathcal{T}_2 \oplus {X}^\mathcal{T}_3\oplus ... \oplus {X}^\mathcal{T}_{N} \}\},  \notag \\
    & ...,\notag \\
    &  {C}(X^\mathcal{T}_N) = \{ X^\mathcal{T}_N\} ,\notag \\
    & {C}(\{X^\mathcal{T}\}_1^N)) = \sum_{i =1}^N {C}(X^\mathcal{T}_i) = \notag \\
    &\{\{X^\mathcal{T}\}_1^N, \underbrace{...}_{\frac{n^2 -N}{2 } \text{terms omitted.}}  \}   \ni \{X^\mathcal{T}\}_1^N .
\end{split}
\end{equation}

Since $ g_o ({X}^\mathcal{S} , \{X^\mathcal{T}\}_1^N ) \in g_x ({X}^\mathcal{S} , \{X^\mathcal{T}\}_1^N )$, we can infer that $g_o$ is a sub-group of $g_x$ and $\chi _{g_o} \in  \chi _{g_x}$. 
\hfill $\Box$
\end{proof}

\begin{figure*}[t]
\centering
\includegraphics[width=0.95\linewidth]{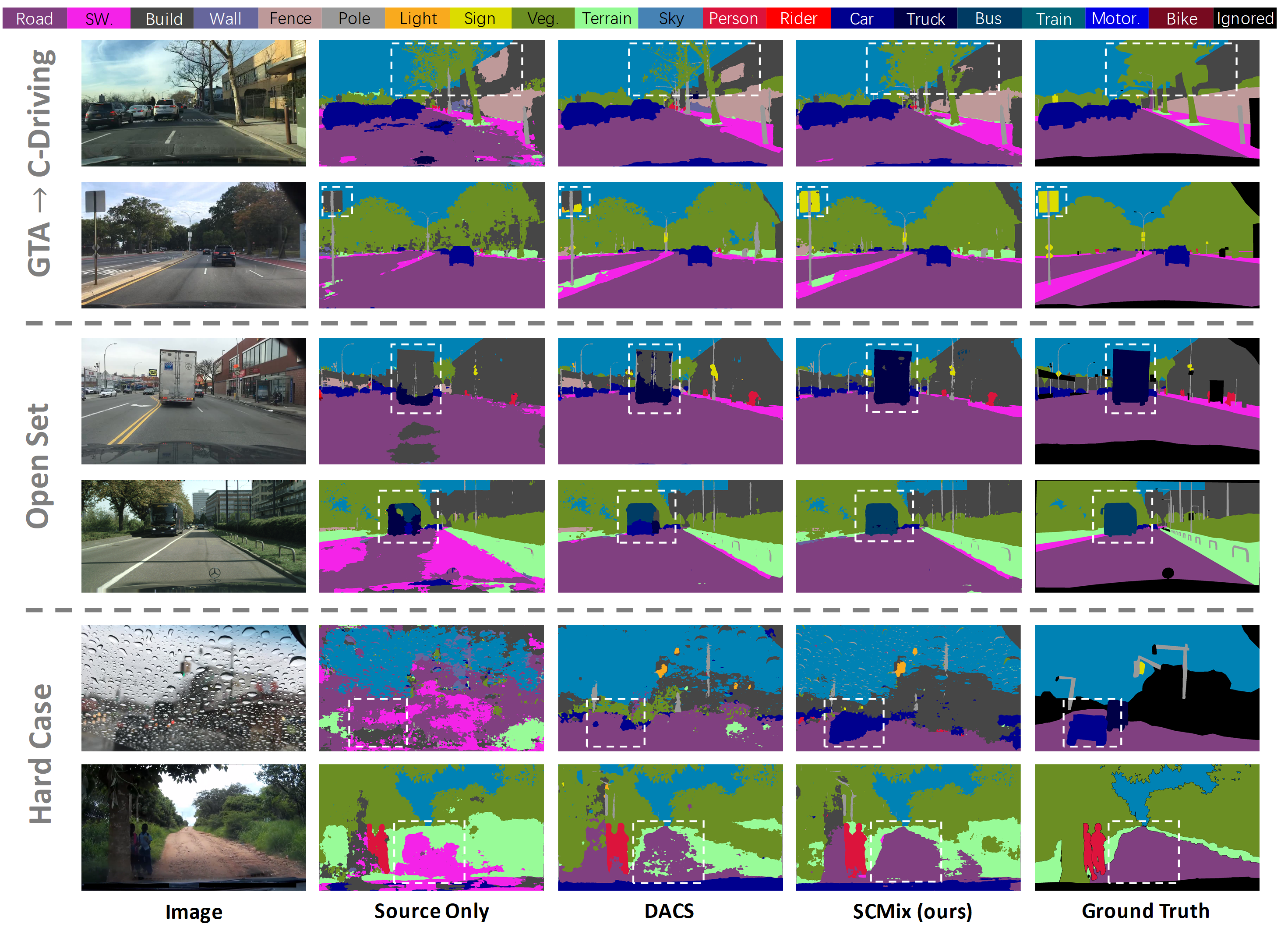}
\caption{Visualization of comparative predictions on GTA $\rightarrow$ C-Driving. }%
\label{fig:sup1}
\end{figure*}

\section{Visualization}
As depicted in Fig.~\ref{fig:sup1}, we present additional comparative results of semantic segmentation for the OCDA task on the GTA $\rightarrow$ C-Driving dataset, utilizing the same SegFormer backbone. Our baseline method, DACS, yields significant improvements in segmentation results and reduces prediction errors compared to the source model, particularly in the target and open unseen domains. Moreover, our approach exhibits better robustness and generalization capabilities when faced with unseen scenes (as indicated by the white dash boxes in the 1st and 2nd rows) or vehicle models (as indicated by the white dash boxes in the 3rd and 4th rows). Additionally, for some unseen factors that were not presented in the training set, such as raindrops on car windows (5th row) and muddy roads (6th row), our method demonstrates excellent performance.

\section{Code and Reproducibility}
Our code framework is built upon the PyTorch-based~\cite{paszke2019pytorch} MMSegmentation Library~\cite{mmseg2020} and we are committed to ensuring that our code is well-structured and highly readable. It has come to our attention that a majority of previous works, with the exception of the pioneering work CDAS~\cite{cdas}, did not make their code publicly available, which harms the development of the field of OCDA. Furthermore, most of our comparative results are derived from the ML-BPM~\cite{mlbpm} method under the same experimental setting. In light of this, and with a view to better contributing to the development of the OCDA community and ensuring the reproducibility of the proposed SCMix, we will release our code upon acceptance of the article.

\end{document}